\documentclass[letterpaper]{article} 
\usepackage{aaai24}  
\usepackage{times}  
\usepackage{helvet}  
\usepackage{courier}  
\usepackage[hyphens]{url}  
\usepackage{graphicx} 
\usepackage{natbib}  
\usepackage{caption} 
\usepackage{dsfont}
\usepackage{epsfig}
\usepackage{graphicx}
\usepackage{amsmath}
\usepackage{amssymb}
\usepackage{booktabs}
\usepackage{subfigure}
\usepackage{graphicx}
\usepackage{multirow}
\usepackage{amsthm}
\usepackage{soul}

\usepackage{algorithm}
\usepackage{algorithmic}

\usepackage{newfloat}
\usepackage{listings}
\DeclareCaptionStyle{ruled}{labelfont=normalfont,labelsep=colon,strut=off} 
\floatstyle{ruled}
\newfloat{listing}{tb}{lst}{}
\floatname{listing}{Listing}
\newtheorem{theorem}{Theorem}

\title{PICNN: A Pathway towards Interpretable Convolutional Neural Networks}
\newcommand\correspondingauthor{\thanks{Corresponding author.}}

\author{
Wengang Guo\equalcontrib,
Jiayi Yang\equalcontrib,
Huilin Yin,
Qijun Chen,
Wei Ye\correspondingauthor
}
\affiliations{
College of Electronic and Information Engineering, Tongji University, Shanghai, China\\
\{guowg, 2111125, yinhuilin, qjchen, yew\}@tongji.edu.cn
}

\usepackage{bibentry}

\begin{document}

\maketitle

\begin{abstract}
	Convolutional Neural Networks (CNNs) have exhibited great performance in discriminative feature learning for complex visual tasks. Besides discrimination power, interpretability is another important yet under-explored property for CNNs. One difficulty in the CNN interpretability is that filters and image classes are entangled. In this paper, we introduce a novel pathway to alleviate the entanglement between filters and image classes. The proposed pathway groups the filters in a late conv-layer of CNN into class-specific clusters. Clusters and classes are in a one-to-one relationship. Specifically, we use the Bernoulli sampling to generate the filter-cluster assignment matrix from a learnable filter-class correspondence matrix. To enable end-to-end optimization, we develop a novel reparameterization trick for handling the non-differentiable Bernoulli sampling. We evaluate the effectiveness of our method on ten widely used network architectures (including nine CNNs and a ViT) and five benchmark datasets. Experimental results have demonstrated that our method PICNN (the combination of standard CNNs with our proposed pathway) exhibits greater interpretability than standard CNNs while achieving higher or comparable discrimination power.
\end{abstract}

\section{Introduction}
The remarkable discrimination power of convolutional neural networks (CNNs) fosters great applications in numerous tasks. However, in safety-critical domains such as autonomous vehicles \cite{zablocki2022explainability} and healthcare \cite{d2022underspecification}, interpretability is another crucial property that needs to be considered.



The interpretability of CNNs has received growing interest in recent research.
Earlier posthoc explanation methods \cite{zeiler2014visualizing,simonyan2013deep,springenberg2014striving,zhou2016learning,selvaraju2017grad,bau2017network} focus on generating offline interpretations such as saliency maps~\cite{simonyan2013deep} and class activation mapping (CAM) ~\cite{zhou2016learning} for well-trained CNNs. But posthoc methods cannot improve the intrinsic interpretability of CNNs, as they operate independently from the training process. Recently, the focus of the community has shifted to train interpretable CNNs.
For example, \cite{zhang2018interpretable,shen2021interpretable} encourage each filter in a late conv-layer to respond to only one object part. Nevertheless, these methods cannot directly reveal concepts learned by filters, as they rely on some posthoc methods~\cite{bau2017network,zhou2015object} to associate filters to the predefined concepts.

\begin{figure}[t]
	\centering
	\includegraphics[width=\columnwidth]{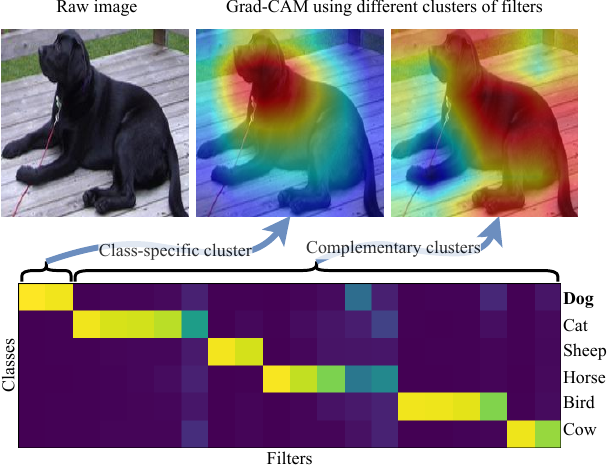}
	\caption{
		Top: comparison of Grad-CAM visualizations using different class-specific clusters of filters. Bottom: the display of the learned filter-class correspondence matrix $\mathbf{P}$. The dataset is PASCAL VOC Part \cite{chen2014detect} with  six animal classes, and the target late conv-layer consists of 20 class-specific filters.\\
		\
	}
	\label{fig_cam_P}
\end{figure}

In this paper, interpretability is characterized by the alignment between filters and class-level concepts that are human-understandable. We propose a novel method to train interpretable CNNs whose filters in the late conv-layer can directly reveal class-level concepts without the help of any posthoc methods.
We focus on filters in the late conv-layer since these filters encode class-level concepts rather than low-level primitives \cite{zeiler2014visualizing}.
\cite{fong2018net2vec} has shown that a CNN uses multiple filters in the late conv-layers to collectively encode a class-level concept.
Considering a target late conv-layer with $N$ filters, there exist $2^N$ possible combinations of filters, allowing these filters to encode up to $2^N $ class-level concepts. This huge concept space prohibits the explanation of the relations between filters and classes due to their complex many-to-many correspondences, known as filter-class entanglement.
As argued in prior work \cite{liang2020training}, filter-class entanglement is one of the most critical reasons that hamper the interpretability of CNNs.

To mitigate filter-class entanglement, we group the $N$ filters into $K$ class-specific clusters, where $K$ is the number of image classes. Clusters and classes are in a one-to-one correspondence relation, i.e., each class-specific cluster is encouraged to contain filters that encode the class-level concepts of only one particular image class. 
A direct way to group filters is to predefine a fixed filter-cluster assignment matrix before training.
However, this results in suboptimal performance, as the optimal number of filters assigned to each class is unknown a priori. 
\cite{shen2021interpretable} employs spectral clustering \cite{shi2000normalized} to group filters during each forward propagation. Differing from our paper, their goal is to make each filter represent a set of interpretable image features, such as a specific object part or image region with a clear meaning.

In this paper, we develop a novel pathway that can group filters into class-specific clusters and be applied to many network architectures. 
Specifically, the pathway introduces a learnable matrix $\mathbf{P}\in\mathbb{R}^{K\times N}$ whose element $p_{y, i}$ represents the probabilistic correspondence between the $i$-th filter and $y$-th class. Based on $\mathbf{P}$, we can use the argmax operation to generate a binary filter-cluster assignment matrix. However, the argmax operation is non-differentiable, impeding end-to-end optimization. 
One potential solution is using Gumbel-Softmax \cite{jang2016categorical} to approximate the argmax operation to enable gradient propagation. Nevertheless, this approximation introduces bias, as Gumbel-Softmax constitutes a biased estimate of the original argmax operation. Such bias can accumulate during the training process and result in suboptimal performance.
To address this challenge, we consider $p_{y,j}$ as the parameter of the Bernoulli distribution, based on which a filter-cluster assignment matrix is sampled from $\mathbf{P}$. Since the sampling is also non-differentiable, we propose a novel reparameterization trick to support end-to-end training. Our filter-cluster assignment strategy yields superior performance compared to the Gumbel-Softmax-based strategy. During end-to-end training, filter grouping and filter learning are jointly optimized:
Filter grouping is conducted in the forward propagation, while filter learning is conducted in the backward propagation. Figure \ref{fig_cam_P} displays the learned correspondence matrix $\mathbf{P}$ and Grad-CAM \cite{selvaraju2017grad} visualizations for validating the learned filter-class correspondence.

The contributions of this paper are as follows:
\begin{itemize}
\item We develop a novel pathway to transform a standard CNN into an interpretable CNN. This pathway is versatile and can be flexibly combined with many CNN architectures and even Transformer architectures.
\item We propose to group filters into class-specific clusters by the filter-cluster assignment matrix, which is sampled from a learnable filter-class correspondence matrix $\mathbf{P}$ by the Bernoulli distribution. Each element of $\mathbf{P}$ can be used as the parameter of the Bernoulli distribution.
\item We propose a reparameterization trick for sampling by the Bernoulli distribution to support end-to-end training.
\item We evaluate the effectiveness of our proposed PICNN using five benchmark datasets. Experimental results demonstrate that PICNN achieves higher or comparable discrimination power and better interpretability than backbone models.
\end{itemize}
\section{Related Work}
We briefly survey the literature on posthoc filter interpretability, learning interpretable filters, and class-specific filters. 

\textbf{Posthoc filter interpretability} has been widely studied, which aims to build the mapping between the abstract patterns of filters in well-trained CNNs and the human-understandable domains such as images \cite{wang2021interpretable}.
Earlier works \cite{zeiler2014visualizing,mahendran2015understanding, yosinski2015understanding} inspect maximal activations of filters across different images.
CAM \cite{zhou2016learning} localizes image regions that are most important for a target class.
CAM has gained much popularity and promotes numerous further studies \cite{selvaraju2017grad,lee2021relevance,hasany2023seg,sarkar2023rl}.
Notably, Grad-CAM \cite{selvaraju2017grad} extends CAM by using gradients to weigh the contribution of each filter to the target class.
Other well-known posthoc methods include guided backpropagation \cite{springenberg2014striving}, saliency maps \cite{simonyan2013deep}, and deconvolutional network \cite{zeiler2014visualizing,dosovitskiy2016inverting}.
Moreover, some works disentangle the filter representations of a well-trained CNN into an explanatory graph \cite{zhang2018interpreting}, a decision tree \cite{zhang2019interpreting}, or textual descriptions \cite{hendricks2016generating,yang2022explaining}.
However, these posthoc methods cannot remove the existing filter-class entanglement in well-trained CNNs and may not faithfully capture what the original CNNs compute \cite{rudin2019stop}. 
In contrast, our work focuses on training interpretable CNNs, where the relations between filters and classes are clearly revealed by the learned filter-class correspondence matrix.

\textbf{Learning interpretable filters}
has been conducted in ICNN \cite{zhang2018interpretable}. The learned each filter represents a specific object part, such as animal eyes. 
Filters in ICNN could only represent object parts in ball-like areas.
To overcome this limitation,  ICCNN \cite{shen2021interpretable} extends filter interpretability to image regions with arbitrary shapes.
However, the filters learned in both ICNN and ICCNN are not class-specific, as they are trained in a class-agnostic manner.
Learning class-specific filters is first introduced by Class-Specific Gate (CSG) \cite{liang2020training}.
Specifically, CSG introduces a binary filter-class gate matrix to represent correspondences between filters and classes.
The binary filter-class gate matrix is relaxed by the $l_1$ norm to weigh feature maps of filters and simultaneously forced to be sparse.
Differing from CSG, we consider the filter-class correspondences as probability parameters and optimize these parameters using a probability approach.

\textbf{Class-specific filters} have been widely utilized across various studies.
Class-specific filters enable the integration of direct supervision into the hidden layers of CNNs, thereby alleviating gradient vanishing \cite{jiang2017learning}.
\cite{wang2018learning,martinez2019action} improve fine-grained recognition using a filter bank that captures class-specific discriminative patches.
Except for discriminative models, generative models also benefit from class-specific formulation.
\cite{tang2020local,li2021collaging} create multiple class-specific generators within a GAN \cite{goodfellow2014generative} to facilitate the generation of small objects.
\cite{kweon2021unlocking} proposes a class-specific adversarial erasing framework to generate more precise CAM.
Unlike these models that enhance the discrimination or generation power of CNNs, we focus on filter interpretability.
Additionally, these models require predefined filter-class correspondences, whereas our work automatically learns filter-class correspondences.


\section{Model PICNN}
This paper aims to train interpretable CNNs whose filters are disentangled from image classes and can directly reveal class-level concepts.
To achieve this, we propose to group filters into class-specific clusters. Each class-specific cluster of filters exclusively activates for inputs from a particular class and deactivates for inputs from other classes.

\subsection{Objective Function}
We enrich the target conv-layer in a CNN, typically the last conv-layer, with an extra pathway, i.e., the interpretation pathway to learn class-specific filters (Figure \ref{fig_fk}).
We preserve the original discrimination pathway to fit the underlying task. 

\textbf{1)} In the discrimination pathway (the blue shaded area in Figure \ref{fig_fk}), the classifier (i.e., the multilayer perceptron (MLP)) receives complete feature maps from the target conv-layer and outputs a softmax-normalized prediction.
The objective is to fit the underlying task by optimizing a loss function, such as the cross-entropy loss for a classification task:
\begin{equation}
	\mathcal{L}_{\text {dis}} = H(\mathbf{y},\mathbf{y}_1)
\end{equation}
where $\mathbf{y} \in \mathbb{R}^K$ is the one-hot encoding of the class label\footnote{In the paper, integer $y\in \{1,\cdots,K\}$ represents the true class label of an image, while vector $\mathbf{y}\in \mathbb{R}^K$ represents the one-hot encoding of $y$.}, $\mathbf{y}_1 \in \mathbb{R}^K$ is the prediction from the discrimination pathway, and $K$ is the number of image classes.

\textbf{2)} In the interpretation pathway (the green shaded area in Figure \ref{fig_fk}), we input into the same classifier used in the discrimination pathway the feature maps of a class-specific cluster of filters. We postulate that each class-specific cluster of filters specialized for a given class is essential for the network to classify that class, whereas other clusters can be removed/masked without compromising discrimination power.
Then, the classifier outputs a prediction $\mathbf{y}_2 \in \mathbb{R}^K$.
The cross-entropy loss in this interpretation pathway is:
\begin{equation}
	\mathcal{L}_{\text {int}} = H(\mathbf{y},\mathbf{y}_2)
	\label{eq_ce_mask}
\end{equation}


The interpretation pathway is optimized to achieve a comparable classification performance as the discrimination pathway. The optimization of the interpretation pathway encourages the learning of the class-specific filters in the discrimination pathway. We simultaneously optimize the losses of these two pathways as follows:
\begin{equation}
	\mathcal{L} =\mathcal{L}_{\text {dis}}+\lambda \mathcal{L}_{\text {int}}
\end{equation}
where $\lambda$ is a regularization parameter.


\subsection{Class-specific Grouping of Filters}
\label{sec_mask}

\subsubsection{Filter-class Correspondence Matrix}
We introduce a learnable matrix $\mathbf{P}\in \mathbb{R}^{K\times N}$ to indicate the probabilities of the filter-class correspondences, where $N$ is the number of filters in the target conv-layer.
A larger value of $p_{y,i}$ represents a stronger correspondence between the $i$-th filter and the $y$-th class.
With the correspondence matrix $\mathbf{P}$, we can generate the filter-cluster assignment matrix $\mathbf{Z}$ by the argmax operation. 
However, this operation is non-differentiable. 
Using Gumbel-Softmax \cite{jang2016categorical} for approximating the argmax operation induces approximation bias and yields suboptimal performance.
Thus, we propose to sample $\mathbf{Z}$ from $\mathbf{P}$ by the Bernoulli distribution parameterized by the element of $\mathbf{P}$.

\subsubsection{Bernoulli Sampling}
Supposing the input belongs to the $y$-th class, we conduct $N$ independent Bernoulli trials to generate a binary filter-cluster assignment vector $\mathbf{z}=[z_{1},...,z_{N}]\in \{0,1\}^N$ from $\mathbf{p}_{y}=[p_{y,1},...,p_{y,N}]$, i.e., $z_{i} \sim \mathbf{Ber}(p_{y,i}), i=1,2,\ldots,N$, where $\mathbf{p}_{y}$ is the $y$-th row of matrix $\mathbf{P}$.
The $i$-th element $z_i$ determines whether the $i$-th filter belongs to (if $z_i = 1$) the class-specific cluster of the $y$-th class or not (if $z_i = 0$).
After that, we use $\mathbf{z}$ to mask the complete feature maps $\mathbf{H}=\{\mathbf{H}_i\}_{i=1}^N$ (each $\mathbf{H}_i$ is a 2D matrix) of all the filters to generate masked feature maps $\widetilde{\mathbf{H}}=\{\widetilde{\mathbf{H}}_i\}_{i=1}^N$ by the Hadamard product:
\begin{equation}
	\widetilde{\mathbf{H}}_i=\mathbf{H}_i \odot z_i \quad i=1,\cdots,N
	\label{eq_mask}
\end{equation}
where $z_i$ is broadcasted along the spatial dimensions of $\mathbf{H}_i$ for compatible product operation. If $z_i = 1$, $\widetilde{\mathbf{H}}_i$ is the original feature maps of the $i$-th filter; If $z_i = 0$, $\widetilde{\mathbf{H}}_i$ is set to a matrix of all zero elements.
Finally, we feed both the complete feature maps $\mathbf{H}$ and the masked ones $\widetilde{\mathbf{H}}$ into the classifier to compute predictions $\mathbf{y}_1$ and $\mathbf{y}_2$, respectively.


\subsection{Optimization}
Now we elaborate on how to optimize the correspondence matrix $\mathbf{P}$ together with other network parameters using stochastic gradient descent (SGD).
The challenge is that sampling the filter-cluster assignment matrix $\mathbf{Z}$ from $\mathbf{P}$ by the Bernoulli trials is non-differentiable, which impedes SGD optimization of $\mathbf{P}$.
To deal with this, we resort to the reparameterization trick \cite{kingmaauto}.

\begin{figure*}[!htp]
	\centering
	\includegraphics[width=\textwidth]{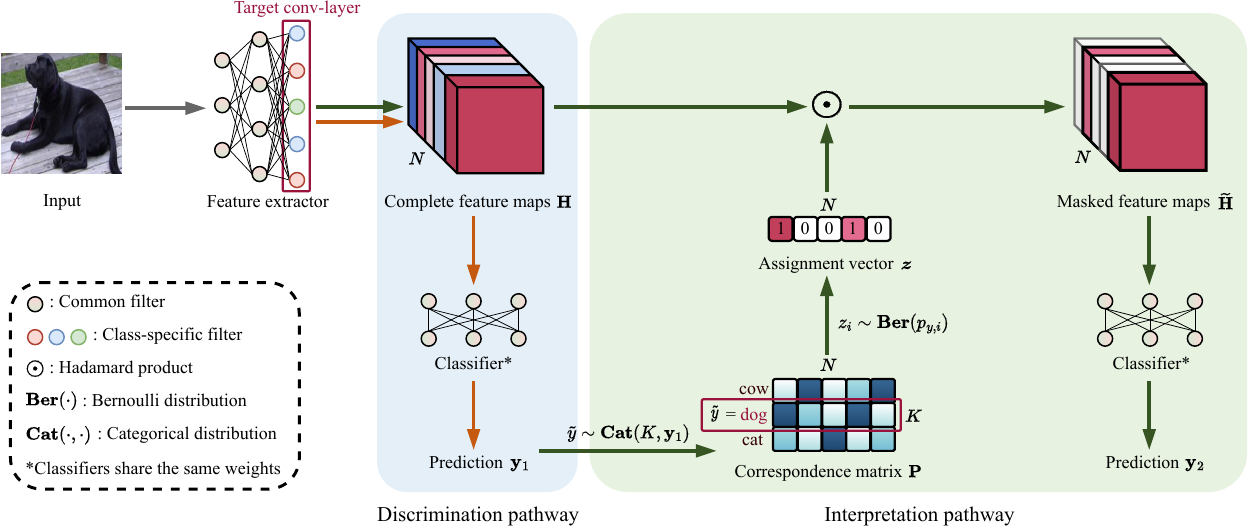}
	\caption{The pipeline of our model PICNN, which consists of \textbf{a)} the discrimination pathway to fit the underlying classification task and \textbf{b)} the interpretation pathway to group filters into class-specific clusters, whose optimization encourages the discrimination pathway to learn class-specific filters.
	}
	\label{fig_fk}
\end{figure*}

To avoid computing the partial derivative for sampling, we propose to reparameterize the assignment vector $\mathbf{z}$, which is initially drawn from the Bernoulli distribution with the probability parameter $\mathbf{p}_{y}$. 
Specifically, we redefine $z_{i} \sim \mathbf{Ber}(p_{y,i})$ as:
\begin{equation}
z_i=p_{y,i}+\eta _i
\label{eq_z}
\end{equation}
where $\eta _i$ is a normalization scalar defined as:
\begin{equation}
	\begin{aligned}
		\eta _i=
		\begin{cases}
			1-p_{y,i},&p_{y,i}\ge\epsilon_i \\
			-p_{y,i},&p_{y,i}<\epsilon_i
		\end{cases}\quad i=1,\cdots,N
	\end{aligned}
	\label{eq_eta}
\end{equation}
where $\epsilon_i$ is sampled from the continuous uniform distribution $\mathbf{\mathcal{U}}(0,1)$, i.e., $\epsilon_i \sim \mathbf{\mathcal{U}}(0,1)$.
This reparameterization trick converts the sampling of $z_i$ from a Bernoulli distribution, which relies on trainable parameters, into sampling $\epsilon_i$ from a fixed uniform distribution. This trick enables us to compute the partial derivative $\frac{\partial \mathbf{z}}{\partial \mathbf{P}}$.

\begin{theorem}
$z_i \sim \mathbf{Ber}(p_{y,i})$.
\end{theorem}

\begin{proof}
Substituting Equation \ref{eq_eta} into Equation \ref{eq_z}, we obtain:
\begin{equation}
	\begin{aligned}
		z_i=
		\begin{cases}
			1,&p_{y,i}\ge\epsilon_i \\
			0,&p_{y,i}<\epsilon_i
		\end{cases} \quad i=1,\cdots,N
	\end{aligned}
\end{equation}
Since $\epsilon_i \sim \mathbf{\mathcal{U}}(0,1)$, the probability $p(p_{y,i}\ge\epsilon_i)$ equals the ratio between the length of the interval $[0,p_{y,i}]$ and the total interval $[0, 1]$.
Thus, we have  $p(z_i=1)=p(p_{y,i}\ge\epsilon_i)=p_{y,i}$ and $p(z_i=0)=1-p(z_i=1)=1-p_{y,i}$,  meaning $z_i$ follows the Bernoulli distribution $\mathbf{Ber}(p_{y,i})$.
\end{proof}

%

Our model may suffer from a trivial solution, as the prediction $\mathbf{y}_2$ is indirectly dependent on the label $y$ (from which a red arrow flows in Figure \ref{fig_trivialsolution}).
As a result, the network can exploit this dependency and effortlessly minimize the loss $\mathcal{L}_{\text {int}}$ (Figure \ref{fig_2D_b}).
To prevent this, we propose to replace the label $y$ with a pseudo-label $\widetilde{y}$ for indexing the correspondence matrix $\mathbf{P}$.
The pseudo-label $\widetilde{y}$ is drawn from a categorical distribution parameterized by the softmax-normalized prediction $\mathbf{y}_1$ from the discrimination pathway, i.e., $\widetilde{y} \sim \mathbf{Cat}(K,\mathbf{y}_1)$.
This intuitive idea is that the pseudo-label $\widetilde{y}$ gradually approaches the label $y$ as the training unfolds.
To support SGD optimization, we utilize the reparameterization trick once again to redefine $\widetilde{y}$. Specifically, we reparameterize the pseudo-label as:
\begin{equation}
\widetilde{\mathbf{y}}=\mathbf{y}_1+\boldsymbol{\tau}
\label{eq_y}
\end{equation}
where $\boldsymbol{\tau}$ is a constant vector defined as:
\begin{equation}
	\begin{aligned}
		\boldsymbol{\tau}[k]=
		\begin{cases}
			1-\mathbf{y}_1[k],&\sum\limits_{j=0}^{k-1}\mathbf{y}_1[j] < \xi  \le \sum\limits_{j=0}^{k}\mathbf{y}_1[j] \\
			-\mathbf{y}_1[k],&\text{otherwise}
		\end{cases}\\ k=1,\cdots,K
	\end{aligned} 
	\label{eq_tau}
\end{equation}
where  $\xi \sim \boldsymbol{\mathcal{U}}(0,1)$, $[k]$ indexes the $k$-th entry of a vector and we append 0 before the first element of $\mathbf{y}_1$ and 1 after the last element of $\mathbf{y}_1$.

\begin{theorem}
$\widetilde{y} \sim \mathbf{Cat}(K,\mathbf{y}_1)$.
\end{theorem}

\begin{proof}
Substituting Equation \ref{eq_tau} into Equation \ref{eq_y}, we obtain:
\begin{equation}
	\begin{aligned}
		\widetilde{\mathbf{y}}[k]=
		\begin{cases}
			1,&\sum\limits_{j=0}^{k-1}\mathbf{y}_1[j] < \xi  \le \sum\limits_{j=0}^{k}\mathbf{y}_1[j] \\
			0,&\text{otherwise}
		\end{cases}
	\end{aligned}
\end{equation}
Since $\xi \sim \mathbf{\mathcal{U}}(0,1)$, the probability $p(\sum_{i=0}^{k-1}\mathbf{y}_1[i] < \xi  \le \sum_{i=0}^{k}\mathbf{y}_1[i])$ equals to the ratio between the length of the interval $[\sum_{i=0}^{k-1}\mathbf{y}_1[i],\sum_{i=0}^{k}\mathbf{y}_1[i]]$ and the total interval $[0, 1]$.
Thus, we have $p(\widetilde{{y}}=k)=\mathbf{y}_1[k]$, i.e.,  $\widetilde{{y}} \sim \mathbf{Cat}(K,\mathbf{y}_1)$.

\end{proof}

\begin{figure}[tbh]
	\begin{center}
		\includegraphics[width=\columnwidth]{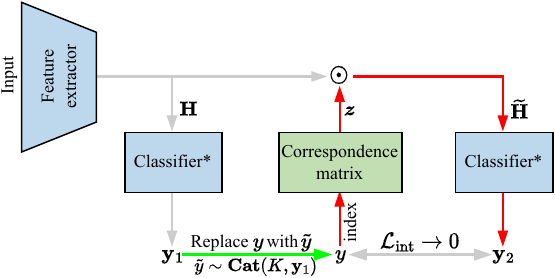}
	\end{center}
	\caption{The interpretation pathway encounters a label leakage issue, as the label $y$ is employed to index the correspondence matrix $\mathbf{P}$. To deal with this issue, we replace $y$ with $\widetilde{y}$ sampled from the categorical distribution parameterized by the softmax-normalized prediction of the $y$-th class in the discrimination pathway.
	}
	\label{fig_trivialsolution}
\end{figure}

\begin{figure}[t]
	\centering
	\subfigure[]{\includegraphics[width=0.32\textwidth]{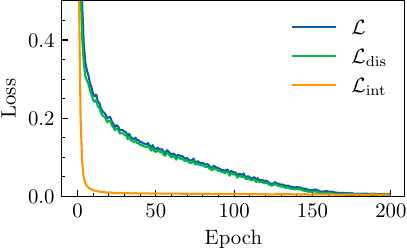}\label{fig_2D_b}}
 
	\subfigure[]{\includegraphics[width=0.32\textwidth]{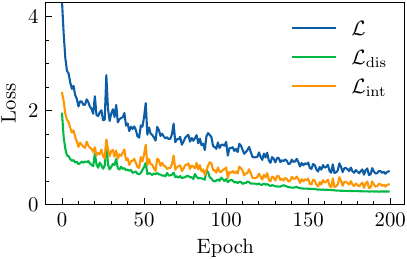}\label{fig_2D_c}}
	\caption{
		Trivial solution. 
		\textbf{a)} Using the label $y$ as index vector, the loss $\mathcal{L}_{\text {int}}$ quickly becomes zero, leading to a trivial solution.
		\textbf{b)} After replacing the label $y$ with the pseudo-label $\widetilde{y}$, our model shows a stable evolution.
	}
	\label{fig_acc_loss}
\end{figure}

Figure \ref{fig_2D_c} demonstrates that this categorical sampling effectively circumvents the trivial solution.

\subsection{Time Complexity Analysis}

Most clustering methods incur a high time complexity between  $\mathcal{O} (M^2)$ and $\mathcal{O} (M^3)$ if the number of input images is $M$. In PICNN, the interpretation pathway plays a role that is equivalent to a clustering component but with a lower time complexity of $\mathcal{O}(M(2K+2K^2+2KN+Nd^2+2KN^2))$, where $K$ is the number of image classes, $N$ is the number of filters and $d$ is the size of the feature map. And the space complexity is $\mathcal{O}(NK)$. 
We report the running time of PICNN for large-scale architectures in Table~\ref{tab_large_model}.
PICNN adds at most 5\% computational overhead to backbone models.



\begin{table*}[!htp]
\setlength\tabcolsep{1.5pt}
\begin{tabular}{@{}cccccccccccccc@{}}
\toprule
{ }       & { }     & \multicolumn{4}{c}{{ CIFAR-10}}                & \multicolumn{4}{c}{{ PASCAL VOC Part}}                                      & \multicolumn{4}{c}{{ STL-10}}         \\ \midrule
\multicolumn{1}{c|}{{Backbone}}          & \multicolumn{1}{c|}{{Method}} & { ACC1$\uparrow$}           & { ACC2$\uparrow$}           & { ACC3$\downarrow$}        & \multicolumn{1}{c|}{{ MIS$\uparrow$}}            & { ACC1$\uparrow$}           & { ACC2$\uparrow$}           & { ACC3$\downarrow$}        & \multicolumn{1}{c|}{{ MIS$\uparrow$}}            & { ACC1$\uparrow$}           & { ACC2$\uparrow$}           & { ACC3$\downarrow$}        & { MIS$\uparrow$}            \\ \midrule
\multicolumn{1}{c|}{{ }}                               & \multicolumn{1}{c|}{{ STD}}             & { 0.952}          & { 0.588}          & { 0.950}          & \multicolumn{1}{c|}{{ 0.191}}          & { 0.782}          & { 0.769}          & { 0.781}          & \multicolumn{1}{c|}{{ 0.203}}          & { 0.895}          & { 0.778}          & { 0.894}          & { 0.148}          \\
\multicolumn{1}{c|}{{ }}                               & \multicolumn{1}{c|}{{ ICCNN}}           & { \textbf{0.955}} & { N/A}              & { N/A}              & \multicolumn{1}{c|}{{ 0.067}}          & { 0.754}          & { N/A}              & { N/A}              & \multicolumn{1}{c|}{{ 0.140}}          & { 0.871}          & { N/A}              & { N/A}              & { 0.052}          \\
\multicolumn{1}{c|}{{ }}                               & \multicolumn{1}{c|}{{ CSG}}             & { 0.830}          & { 0.505}          & { 0.659}          & \multicolumn{1}{c|}{{ 0.132}}          & { 0.781}          & { 0.686}          & { 0.666}          & \multicolumn{1}{c|}{{ 0.148}}          & { 0.888}          & { 0.373}          & { 0.699}          & { 0.115}          \\
\multicolumn{1}{c|}{\multirow{-4}{*}{{ ResNet-18}}}       & \multicolumn{1}{c|}{{ PICNN}}           & { 0.951}          & { \textbf{0.948}} & { \textbf{0.333}} & \multicolumn{1}{c|}{{ \textbf{0.243}}} & { \textbf{0.800}} & { \textbf{0.798}} & { \textbf{0.639}} & \multicolumn{1}{c|}{{ \textbf{0.246}}} & { \textbf{0.896}} & { \textbf{0.893}} & { \textbf{0.657}} & { \textbf{0.224}} \\ \midrule\midrule
\multicolumn{1}{c|}{{ }}                               & \multicolumn{1}{c|}{{ STD}}             & { \textbf{0.918}} & { 0.401}          & { 0.913}          & \multicolumn{1}{c|}{{ 0.128}}          & { 0.803}          & { 0.629}          & { 0.799}          & \multicolumn{1}{c|}{{ 0.138}}          & { 0.702}          & { 0.684}          & { 0.697}          & { \textbf{0.142}} \\
\multicolumn{1}{c|}{\multirow{-2}{*}{{ VGG-11}}}       & \multicolumn{1}{c|}{{ PICNN}}           & { 0.915}          & { \textbf{0.910}} & { \textbf{0.192}} & \multicolumn{1}{c|}{{ \textbf{0.177}}} & { \textbf{0.840}} & { \textbf{0.837}} & { \textbf{0.346}} & \multicolumn{1}{c|}{{ \textbf{0.166}}} & { \textbf{0.716}} & { \textbf{0.709}} & { \textbf{0.316}} & { 0.131}          \\  \midrule
\multicolumn{1}{c|}{{ }}                               & \multicolumn{1}{c|}{{ STD}}             & { 0.904}          & { 0.308}          & { 0.893}          & \multicolumn{1}{c|}{{ 0.091}}          & { \textbf{0.838}} & { 0.733}          & { 0.837}          & \multicolumn{1}{c|}{{ 0.122}}          & { \textbf{0.646}} & { 0.618}          & { 0.644}          & { 0.115}          \\
\multicolumn{1}{c|}{\multirow{-2}{*}{{ AlexNet}}}      & \multicolumn{1}{c|}{{ PICNN}}           & { \textbf{0.905}} & { \textbf{0.900}} & { \textbf{0.188}} & \multicolumn{1}{c|}{{ \textbf{0.191}}} & { 0.821}          & { \textbf{0.817}} & { \textbf{0.432}} & \multicolumn{1}{c|}{{ \textbf{0.169}}} & { 0.645}          & { \textbf{0.641}} & { \textbf{0.205}} & { \textbf{0.118}} \\ \midrule
\multicolumn{1}{c|}{{ }}                               & \multicolumn{1}{c|}{{ STD}}             & { \textbf{0.958}} & { 0.438}          & { 0.954}          & \multicolumn{1}{c|}{{ 0.163}}          & { \textbf{0.918}} & { \textbf{0.917}} & { \textbf{0.714}} & \multicolumn{1}{c|}{{ 0.148}}          & { 0.829}          & { 0.791}          & { 0.824}          & { 0.200}          \\
\multicolumn{1}{c|}{\multirow{-2}{*}{{ DenseNet-121}}}     & \multicolumn{1}{c|}{{ PICNN}}           & { 0.957}          & { \textbf{0.957}} & { \textbf{0.462}} & \multicolumn{1}{c|}{{ \textbf{0.257}}} & { 0.917}          & { 0.915}          & { 0.763}          & \multicolumn{1}{c|}{{ \textbf{0.229}}} & { \textbf{0.832}} & { \textbf{0.831}} & { \textbf{0.734}} & { \textbf{0.266}} \\ \midrule
\multicolumn{1}{c|}{{ }}                               & \multicolumn{1}{c|}{{ STD}}             & { 0.950}          & { 0.540}          & { 0.948}          & \multicolumn{1}{c|}{{ 0.208}}          & { \textbf{0.930}} & { 0.844}          & { 0.928}          & \multicolumn{1}{c|}{{ 0.176}}          & { 0.829}          & { 0.818}          & { 0.828}          & { 0.210}          \\
\multicolumn{1}{c|}{\multirow{-2}{*}{{ MobileNetV2}}}    & \multicolumn{1}{c|}{{ PICNN}}           & { \textbf{0.956}} & { \textbf{0.952}} & { \textbf{0.403}} & \multicolumn{1}{c|}{{ \textbf{0.261}}} & { 0.928}          & { \textbf{0.925}} & { \textbf{0.844}} & \multicolumn{1}{c|}{{ \textbf{0.251}}} & { \textbf{0.839}} & { \textbf{0.834}} & { \textbf{0.746}} & { \textbf{0.258}} \\ \midrule
\multicolumn{1}{c|}{{ }}                               & \multicolumn{1}{c|}{{ STD}}             & { 0.955}          & { 0.351}          & { 0.951}          & \multicolumn{1}{c|}{{ 0.174}}          & { \textbf{0.942}} & { 0.889}          & { 0.940}          & \multicolumn{1}{c|}{{ 0.165}}          & { 0.836}          & { 0.820}          & { 0.828}          & { 0.229}          \\
\multicolumn{1}{c|}{\multirow{-2}{*}{{ EfficientNet-B0}}} & \multicolumn{1}{c|}{{ PICNN}}           & { \textbf{0.956}} & { \textbf{0.955}} & { \textbf{0.380}} & \multicolumn{1}{c|}{{ \textbf{0.253}}} & { 0.940}          & { \textbf{0.937}} & { \textbf{0.873}} & \multicolumn{1}{c|}{{ \textbf{0.253}}} & { \textbf{0.842}} & { \textbf{0.833}} & { \textbf{0.764}} & { \textbf{0.275}} \\ \bottomrule 
\end{tabular}
\caption{Experimental results on three benchmark datasets. We conduct comparison experiments between standard methods and our method on six CNN architectures. STD stands for standard CNNs, ICCNN stands for Compositional CNN, CSG stands for Class-Specific Gate CNN, and PICNN is our method. A higher value is better for ACC1, ACC2, and MIS, while a lower value is better for ACC3.}
\label{tab_method_comparison}                                                                                                                       
\end{table*}
\section{Experiments}
\subsection{Experimental Settings}
\subsubsection{Evaluation Metric}
We use classification accuracy and mutual information score (MIS)~\cite{liang2020training} as evaluation metrics to assess the discrimination power and interpretability. We report the following four metrics: \textbf{1)} classification accuracy of the discrimination pathway (ACC1$\uparrow$). 
ACC1 evaluates the discrimination power.
\textbf{2)} classification accuracy of the interpretation pathway (ACC2$\uparrow$).
Since the input to the interpretation pathway is the feature maps of a class-specific cluster of filters, ACC2 is consistently lower than ACC1. 
\textbf{3)} classification accuracy using all the filters except a class-specific cluster of filters (ACC3$\downarrow$).
Since the filters used in the computation of ACC3 and ACC2 complement each other, the lower ACC3 is, the better the model. 
\textbf{4)} MIS (MIS $\uparrow$) measures the mutual information between filter activation and prediction on classes.
MIS is defined as $\text{MIS}=\mathbf{mean}_i(\mathbf{max}_y(m_{yi}))$, where ${m}_{yi}=\mathbf{MI}(\mathbf{H}_i \| \mathbf{y})$.
The metrics ACC2, ACC3, and MIS quantify the degree of filter-class entanglement from different aspects, thus they can evaluate the interpretability.

\subsubsection{Datasets and Types of Backbone CNNs} We use three benchmark classification datasets in Table \ref{tab_method_comparison}, including CIFAR-10 \cite{krizhevsky2009learning}, STL-10 \cite{coates2011analysis}, and PASCAL VOC Part \cite{chen2014detect}. 
CIFAR-10 consists of 50,000 training images and 10,000 test images in 10 classes. STL-10 contains 5,000 training images and 8,000 test images in 10 classes. Following \cite{liang2020training}, we select six animal classes from PASCAL VOC Part with a 70\%/30\% training/test split. 
To further evaluate the efficacy and effectiveness of PICNN on large datasets with more classes, we use CIFAR-100 \cite{krizhevsky2009learning} and TinyImageNet \cite{deng2009imagenet}
in Table \ref{tab_largedataset}. Like CIFAR-10, CIFAR-100 also consists of 50,000 training images and 10,000 test images evenly distributed into 100 classes. 
TinyImageNet is a scaled-down version of the original ImageNet involving 200 classes with a 10:1 ratio of training images to test images. We use the official training/test data split, except for PASCAL VOC Part.

We combine our interpretation pathway with six typical CNN architectures, three large CNN architectures, and a transformer architecture \cite{dosovitskiy2020image} (ViT-b-12) to make them interpretable. 
The six CNN architectures are VGG-11 \cite{simonyan2014very}, AlexNet \cite{krizhevsky2014one}, ResNet-18 \cite{he2016deep}, DenseNet-121 \cite{huang2017densely}, MobileNetV2 \cite{sandler2018mobilenetv2}, and EfficientNet-B0 \cite{tan2019efficientnet}. 
The three deeper and wider CNN architectures include ResNet-50 \cite{he2016deep}, ResNet-152 \cite{he2016deep}, and Wide-ResNet \cite{zerhouni2017wide}.

\subsubsection{Implementation Details}
Our code is based on the PyTorch \cite{paszke2019pytorch} toolbox and publicly available at Github\footnote{\url{https://github.com/spdj2271/PICNN}}. We make filters in the target conv-layer class-specific by optimizing our loss function. 
The regularization parameter $\lambda$ is set to 2 and the effect of the $\lambda$ values is discussed later. Other default settings include: a batch size of 128; the Adam optimizer with an initial learning rate of 0.001; pretrained weights from ImageNet \cite{deng2009imagenet}; and a total of 200 training epochs.
All metrics presented in this paper are computed on the test set. 
The experiments are carried out on a server with an Xeon(R) Platinum 8352V CPU and one Nvidia RTX 4090 GPU.

\subsection{Experimental Results}

\subsubsection{Comparison with Existing Methods}\label{sec:comp}
We first compare our PICNN with a Standard CNN (STD) (ResNet-18, $\lambda=0$), Interpretable Compositional CNNs (ICCNN) \cite{shen2021interpretable} and Class-Specific Gate (CSG) \cite{liang2020training}. All the networks use the same backbone ResNet-18. For STD, we randomly initialize the correspondence matrix $\mathbf{P}$ and use our Bernoulli sampling to generate the filter-cluster assignment matrix. As shown in the upper part of Table \ref{tab_method_comparison}, PICNN significantly outperforms the comparison methods in terms of ACC2, ACC3, and MIS across all datasets. Since the filters learned by ICCNN are not class-specific, ACC2 and ACC3 are not suitable to evaluate its performance. The corresponding result is denoted as N/A.
Additionally, PICNN is better than STD on PASCAL VOC Part and STL-10 datasets in terms of ACC1 and achieves comparable results to STD on CIFAR-10. CSG shows a noticeable decrease in ACC1 on CIFAR-10, and ICCNN shows relatively poorer performance on PASCAL VOC Part and STL-10. The results on large datasets are shown in Table~\ref{tab_largedataset} and demonstrate once again that PICNN has a significant improvement in ACC2, ACC3, and MIS over STD. 

All these results indicate that PICNN achieves better interpretability and maintains comparable discrimination power simultaneously.
In contrast, the comparison methods are unsatisfactory in both the discrimination power and interpretability.
\begin{table}[!htp]
\setlength\tabcolsep{1.5pt}
	\begin{tabular}{l|c|c|cccc}
		\toprule
		Dataset &  Class & Model & ACC1$\uparrow$ & ACC2$\uparrow$ & ACC3$\downarrow$ & MIS$\uparrow$ \\ \midrule
		\multirow{2}{*}{CIFAR-100} & \multirow{2}{*}{100} & STD & \textbf{0.755} & 0.039 & 0.753 & 0.020 \\
		&  & PICNN & 0.724 & \textbf{0.707} & \textbf{0.131} & \textbf{0.026} \\ 
  \midrule
  \multirow{2}{*}{TinyImageNet} & \multirow{2}{*}{200} & STD & \textbf{0.619} & 0.010 & 0.619 & 0.010 \\
		&  & PICNN & 0.555 & \textbf{0.514}& \textbf{0.082} & \textbf{0.017}  \\ \midrule
	\end{tabular}
	\caption{Performance on large multi-class datasets with ResNet-18 as the backbone.}
	\label{tab_largedataset}
\end{table}

\subsubsection{Various Network Architectures}
One advantage of our method lies in its versatility to combine with various network architectures.
As shown in Table \ref{tab_method_comparison} and Table \ref{tab_large_model}, our novel pathway significantly boosts the interpretability of all the ten network architectures, which include six typical CNN architectures, three large CNN architectures, and a transformer architecture.
Except for using DenseNet-121 as the backbone on the PASCAL VOC Part dataset, PICNN has the best performance in terms of all three interpretability metrics ACC2, ACC3, and MIS. Compared with STD, the improvement is significant. PICNN does not work well when DenseNet-121 is the backbone. One reason might be: each layer in DenseNet-121 is connected to all other layers. 
As a result, the target conv-layer may include too many low-level concepts from early conv-layers, which leads to difficulties in grouping filters into class-specific clusters. 
We can also see that PICNN achieves higher ACC1 values in 14 out of 21 experiments. Besides, the performance of PICNN with ViT-b-12 as the backbone demonstrates that our approach is not only applicable to CNN architectures but also to transformer architectures.
\begin{table}[!htp]
\setlength\tabcolsep{1.5pt}
	\begin{tabular}{@{}l|c|cccc|cr@{}}
		\toprule
		Backbone & Metric & ACC1$\uparrow$ & ACC2$\uparrow$ & ACC3$\downarrow$ & MIS$\uparrow$ & Train~~ & ~~Infer \\ \midrule
		ResNet & STD & 0.951 & 0.457 & 0.948 & 0.115 & 18302 & 1564 \\
		-50 & PICNN & \textbf{0.955} & \textbf{0.950} & \textbf{0.318} & \textbf{0.231} & 18484 & 1592 \\ \midrule
		ResNet & STD & 0.957 & 0.481 & 0.954 & 0.150 & 27339 & 2403 \\
		-152 & PICNN & \textbf{0.958} & \textbf{0.954} & \textbf{0.309} & \textbf{0.230} & 27377 & 2514 \\ \midrule
  Wide- & STD & \textbf{0.940} & 0.295 &0.930 & 0.117 & 76572 & 4733 \\
	ResNet	& PICNN & \textbf{0.940} & \textbf{0.933} & \textbf{0.137} & \textbf{0.204} & 77212 & 4831 \\ \midrule
 \midrule
		ViT & STD & 0.799 & 0.242 & 0.794 & 0.066 & 37245 & 3234 \\
	-b-12	& PICNN & \textbf{0.802} & \textbf{0.746} & \textbf{0.039} & \textbf{0.096} & 37615 & 3257 \\ \bottomrule
	\end{tabular}
	\caption{Evaluation for large CNNs and Vision Transformer architecture on the CIFAR-10 dataset. 
    The final two columns display the training and inference time (in seconds).}
	\label{tab_large_model}
\end{table}

\subsubsection{Effectiveness of the Bernoulli Sampling}
Instead of using our Bernoulli sampling to generate the filter-cluster assignment matrix, we can apply Gumbel-Softmax to the correspondence matrix $\mathbf{P}$. 
We conducted an ablation study using Gumbel-Softmax (temperature is set to 0.01) to replace the Bernoulli sampling in PICNN. It can be seen from Table \ref{Table_Gumbel} that the results of the Bernoulli sampling are better than those of Gumbel-Softmax, especially using the metric ACC3. One possible explanation is that Gumbel-Softmax introduces random Gumbel noise to approximate the argmax operation for enabling gradient backpropagation. The approximation affects the results.

\begin{table}[!tbhp]
\setlength\tabcolsep{2.5pt}
\centering
    \begin{tabular}{c|cccc}
    \toprule
    { Model} & { ACC1$\uparrow$} & { ACC2$\uparrow$} & { ACC3$\downarrow$}   & { MIS$\uparrow$}    \\
    \midrule
    {Gumbel-Softmax} & { 0.950} & { 0.946} & { 0.910} & { 0.240}\\
    {Bernoulli-Samping}  & { \textbf{0.951}} & { \textbf{0.948}} & { \textbf{0.333}}  & { \textbf{0.243}}\\
    \bottomrule
    \end{tabular}
    \caption{Performance comparison of PICNN using different strategies for generating filter-cluster assignment matrix on the CIFAR-10 dataset. ResNet-18 is used as the backbone of PICNN.}
\label{Table_Gumbel}
\end{table}

\subsubsection{Effect of the filter-to-class ratio}

For a target conv-layer in a given CNN, we set the filter-to-class ratio to 12.8. This means that for CIFAR-10, we choose the number of filters to be 128. 
To exploit the effect of the filter-to-class ratio $r$, we vary $r=N/K$ on the CIFAR-10 dataset, where $N$ is the number of filters in the target conv-layer and $K$ is the number of image classes. As shown in Figure~\ref{fig_r}, PICNN shows consistent advantages over STD (ResNet-18) across different $r$ values. This indicates that PICNN is robust to the filter-to-class ratio.

\begin{figure}[!htbp]
	\begin{center}
		\includegraphics[width=0.49\columnwidth]{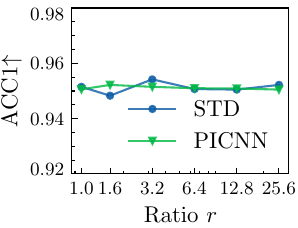}
		\includegraphics[width=0.49\columnwidth]{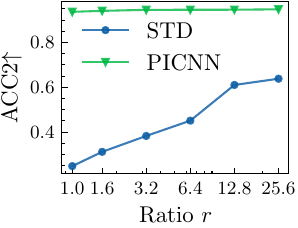}\\
		\includegraphics[width=0.49\columnwidth]{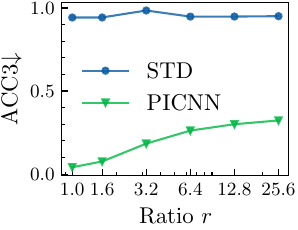}
		\includegraphics[width=0.49\columnwidth]{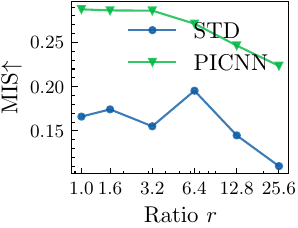}
	\end{center}
	\caption{Effect of the filter-to-class ratio $r={N}/{K}$ on the performance. The backbone (STD) of PICNN is ResNet-18.}
	\label{fig_r}
\end{figure} 


\subsubsection{Effect of Regularization Parameter $\lambda$}
Figure \ref{effect_of_lambda} plots the curves of ACC1, ACC2, ACC3, and MIS of PICNN using ResNet-18 as the backbone on the CIFAR-10 dataset when varying the values of $\lambda$ from 0 to 10. ACC3 decreases steadily, MIS increases fluctuately, while ACC1 and ACC2 remain stable as $\lambda$ increases.
This implies that with the increase of interpretation pathway weights, the filters at the target conv-layer become more class-specific and the discrimination power of the discrimination pathway is preserved.
\begin{figure}[!htbp]
	\begin{center}
		\includegraphics[width=0.45\textwidth]{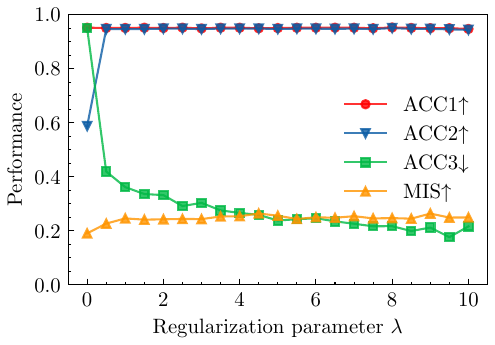}
	\end{center}
	\caption{
		Effect of regularization parameter $\lambda$ of PICNN using ResNet-18 as the backbone on the CIFAR-10 dataset. PICNN works well across a wide range of $\lambda$.
	}
	\label{effect_of_lambda}
\end{figure}

\begin{figure*}[t]
	\begin{center}
	\includegraphics[width=\textwidth]{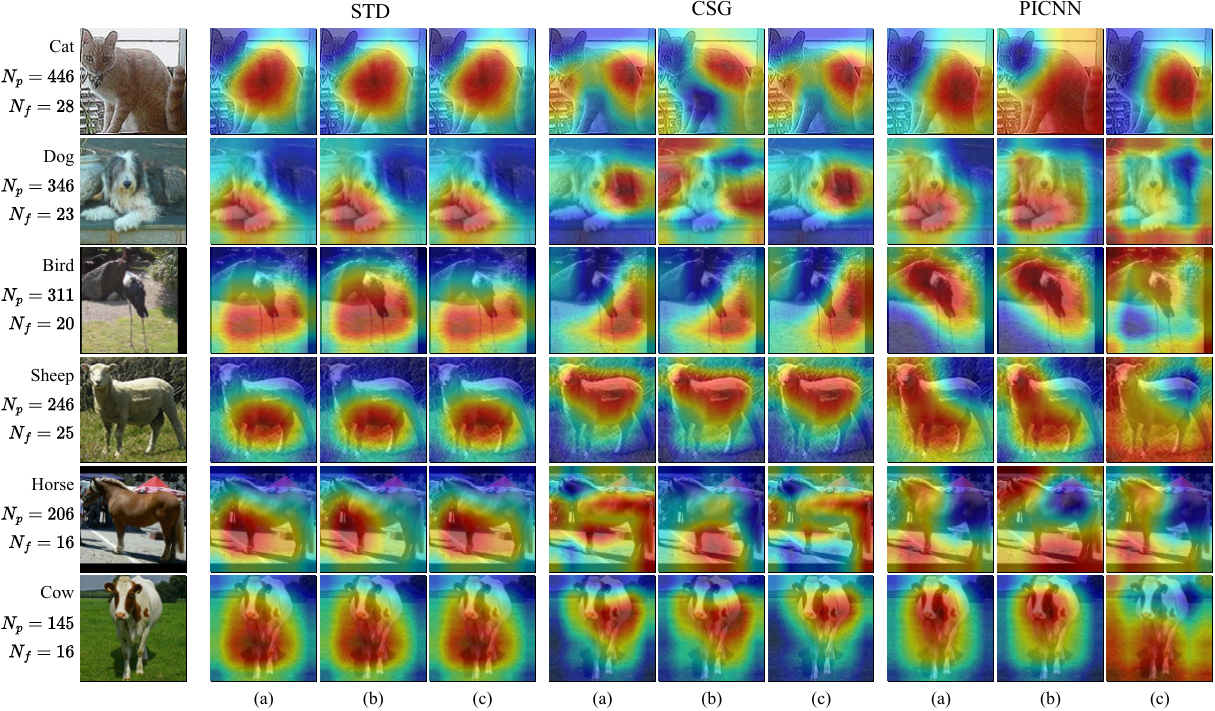}
	\end{center}
	\caption{
 The Grad-CAM visualizations of filters on the PASCAL VOC Part dataset learned by STD (ResNet-18), CSG, and PICNN. CSG and PICNN use the same STD (ResNet-18) as the backbone. In each method, the (a) column represents CAMs using all filters, (b) column represents CAMs using a cluster of class-specific filters, and (c) column represents CAMs using the complementary clusters of class-specific filters. $N_p$ is the number of pictures in a class and $N_f$ is the number of class-specific filters assigned to a class.
	}
	\label{fig_cam}
\end{figure*} 

\subsubsection{Visualization}
Figure \ref{fig_cam} displays the Grad-CAM visualizations on the PASCAL VOC Part dataset using three different sets of filters: (a) all filters, (b) a cluster of class-specific filters, and (c) the complementary clusters of class-specific filters.
As above, we randomly initialize the correspondence matrix $\mathbf{P}$ and use our Bernoulli sampling to generate the filter-cluster assignment matrix for STD (ResNet-18). We observe few differences among these three Grad-CAM visualizations in STD. This phenomenon proves that the filters and classes are entangled in STD. 
The Grad-CAM visualizations using the second set of filters (column (b)) of PICNN sometimes capture more class-specific information than using the first set of filters (column (a)), such as highlighting more parts of the cat's body and the horse's head. The Grad-CAM visualizations using the third set of filters (column (c)) of PICNN mostly highlight unimportant image regions, such as the meadow and road. We can see that the Grad-CAM visualizations using the second and third sets of filters of PICNN differ a lot. But we do not see a significant difference in STD and CSG.

\section{Conclusion}
In this paper, we have proposed a novel pathway to transform a standard CNN into an interpretable CNN without compromising its high discrimination power.
The proposed pathway uses the Bernoulli sampling to generate the filter-cluster assignment matrix from a learnable filter-class correspondence matrix. The filter-cluster assignment matrix groups the filters in the target late conv-layer in CNN into class-specific clusters and thus mitigates the filter-class entanglement problem. Because the Bernoulli sampling is non-differentiable, we propose a reparameterization trick for end-to-end learning. Experiments have shown that our method PICNN is superior to standard CNNs in terms of both interpretability and discrimination power. Moreover, our pathway has good versatility and can be combined with various network architectures.

\section{Acknowledgments}
We thank the anonymous reviewers for their valuable and constructive comments. This work was supported in part by the National Key Research and Development Program of China under Grant 2020AAA0108100.

\bibliography{reference}

\end{document}